\documentclass{article}

\usepackage{microtype}
\usepackage{graphicx}
\usepackage{pstricks}
\usepackage{booktabs} 
\usepackage{import}

\usepackage{hyperref}
\usepackage{url}
\usepackage[accepted]{icml2018}

\usepackage{xfunctions}

\usepackage[utf8]{inputenc} 
\usepackage[T1]{fontenc}    
\usepackage{dsfont}
\usepackage{amsfonts,amsmath,amssymb}       
\usepackage{nicefrac}       
\usepackage{bm}
\usepackage{tikz}
\usepackage{amsthm}
\usepackage{multirow}

\usepackage{siunitx}
\newtheorem{lemma}{Lemma}

\newtheorem{proposition}{Proposition}

\newtheorem{remark}{Remark}

\newcommand{\tensor}[1]{{\bm #1}}

\newcommand{\CP}{{CP}\xspace}

\newcommand{\complex}{{ComplEx}\xspace}
\newcommand{\distmult}{{DistMult}\xspace}
\newcommand{\conve}{{ConvE}\xspace}
\newcommand{\transe}{{TransE}\xspace}

\newcommand{\fb}{{FB15K}\xspace}
\newcommand{\fbd}{{FB15K-237}\xspace}
\newcommand{\wn}{{WN18}\xspace}
\newcommand{\wnrr}{{WN18RR}\xspace}
\newcommand{\yago}{{YAGO3-10}\xspace}

\newcommand{\tens}[1]{\tensor{#1}}
\newcommand{\tensi}[2]{\tensor{#1}_{#2}}

\newcommand{\rkv}{r}
\newcommand{\rk}{R}
\newcommand{\dv}{d}

\newcommand{\nents}{N}
\newcommand{\nrels}{P}

\newcommand{\unkn}{X}
\newcommand{\pred}{\hat{\unkn}}
\newcommand{\tgt}{Y}

\newcommand{\indices}{{\cal S}}
\newcommand{\indecesz}{{|\indices|}}

\newcommand{\loss}{\ell}

\newcommand{\npnorm}[1]{\norm{#1}_{*,p}}

\newcommand{\Co}{\mathbb{C}}

\newcommand{\facv}[2]{u^{(#1)}_{#2}}
\newcommand{\weightv}[1]{w_{\indices}^{(#1)}}
\newcommand{\weightvi}[1]{w_{\indices, i}^{(#1)}}

\newcommand{\nv}[1]{{n_{#1}}}

\newcommand{\allfaccp}{\mathcal{U}_R(\tens{\unkn})}
\newcommand{\nallfaccp}[1]{\overline{\mathcal{U}}^{#1}_R(\tens{\unkn})}

\newcommand{\opnorm}[1]{{\left\vert\kern-0.25ex\left\vert\kern-0.25ex\left\vert #1 
    \right\vert\kern-0.25ex\right\vert\kern-0.25ex\right\vert}}
    
    \DeclareMathOperator{\Tr}{Tr}

\icmltitlerunning{Canonical Tensor Decomposition for Knowledge Base Completion}

\begin{document}

\twocolumn[
\icmltitle{Canonical Tensor Decomposition for Knowledge Base Completion}


\begin{icmlauthorlist}
\icmlauthor{Timoth\'{e}e Lacroix}{fair,enpc}
\icmlauthor{Nicolas Usunier}{fair}
\icmlauthor{Guillaume Obozinski}{enpc}
\end{icmlauthorlist}

\icmlaffiliation{fair}{Facebook AI Research, Paris, France}
\icmlaffiliation{enpc}{Université Paris-Est, Equipe Imagine, 
LIGM (UMR8049) 
Ecole des Ponts ParisTech
Marne-la-Vallée, France}

\icmlcorrespondingauthor{Lacroix Timothee}{timothee.lax@gmail.com}

\icmlkeywords{Knowledge Base Completion, Tensor Norms}

\vskip 0.3in
]


\printAffiliationsAndNotice{}

\begin{abstract}
The problem of Knowledge Base Completion can be framed as a 3rd-order binary tensor completion problem. In this light, the Canonical Tensor Decomposition (\CP) \citep{hitchcock_expression_1927} seems like a natural solution; however, current implementations of \CP on standard Knowledge Base Completion benchmarks are lagging behind their competitors. In this work, we attempt to understand the limits of \CP for knowledge base completion. First, we motivate and test a novel regularizer, based on tensor nuclear $p$-norms. Then, we present a reformulation of the problem that makes it invariant to arbitrary choices in the inclusion of predicates or their reciprocals in the dataset. These two methods combined allow us to beat the current state of the art on several datasets with a \CP decomposition, and obtain even better results using the more advanced \complex model.
\end{abstract}

\section{Introduction}
In knowledge base completion, the learner is given triples (subject, predicate, object) of facts about the world, and has to infer new triples that are likely but not yet known to be true. This problem has attracted a lot of attention \citep{nickel_review_2016,nguyen_overview_2017} both as an example application of large-scale tensor factorization, and as a benchmark of learning representations of relational data. 

The standard completion task is link prediction, which consists in answering queries (subject, predicate, ?) or (?, predicate, object). In that context, the canonical decomposition of tensors (also called CANDECOMP/PARAFAC or \CP) \citep{hitchcock_expression_1927} is known to perform poorly compared to more specialized methods. For instance, \distmult \citep{yang_embedding_2014}, a particular case of \CP which shares the factors for the subject and object modes, was recently shown to have state-of-the-art results \citep{kadlec_knowledge_2017}. This result is surprising because \distmult learns a tensor that is symmetric in the subject and object modes, while the datasets contain mostly non-symmetric predicates.

The goal of this paper is to study whether and how \CP can perform as well as its competitors. To that end, we evaluate three possibilities. 

First, as \citet{kadlec_knowledge_2017} showed that performances for these tasks are sensitive to the loss function and optimization parameters, we re-evaluate \CP with a broader parameter search and a multiclass log-loss. 

Second, since the best performing approaches are less expressive than \CP, we evaluate whether regularization helps. On this subject, we show that the standard regularization used in knowledge base completion does not correspond to regularization with a tensor norm. We then propose to use tensor nuclear $p$-norms \citep{friedland_nuclear_2014}, with the goal of designing more principled regularizers. 

Third, we propose a different formulation of the objective, in which we model separately predicates and their inverse: for each predicate ${\rm pred}$, we create an inverse predicate ${\rm pred}^{-1}$ and create a triple $({\rm obj}, {\rm pred}^{-1}, {\rm sub})$ for each training triple $({\rm sub}, {\rm pred}, {\rm obj})$. At test time, queries of the form $(?, {\rm pred}, {\rm obj})$ are answered as $({\rm obj}, {\rm pred}^{-1}, ?)$. Similar formulations were previously used by \citet{shen2016implicit} and \citet{joulin2017fast}, but for different models for which there was no clear alternative, so the impact of this reformulation has never been evaluated.

To assess whether the results we obtain are specific to \CP, we also carry on the same experiments with a state-of-the-art model, \complex \citep{trouillon_complex_2016}. \complex has the same expressivity as \CP in the sense that it can represent any tensor, but it implements a specific form of parameter sharing. We perform all our experiments on $5$ common benchmark datasets of link prediction in knowledge bases. 

Our results first confirm that within a reasonable time budget, the performance of both \CP and \complex are highly dependent on optimization parameters. With systematic parameter searches, we obtain better results for \complex than what was previously reported, confirming its status as a state-of-the-art model on all datasets. For \CP, the results are still way below its competitors. 

Learning and predicting with the inverse predicates, however, changes the picture entirely. First, with both \CP and \complex, we obtain significant gains in performance on all the datasets. More precisely, we obtain state-of-the-art results with \CP, matching those of \complex. For instance, on the benchmark dataset \fb \citep{bordes_translating_2013}, the mean reciprocal rank of vanilla \CP and vanilla \complex are $0.40$ and $0.80$ respectively, and it grows to $0.86$ for both approaches when modeling the inverse predicates. 

Finally, the new regularizer we propose based on the nuclear $3$-norm, does not dramatically help \CP, which leads us to believe that a careful choice of regularization is not crucial for these \CP models. Yet, for both \CP and \complex with inverse predicates, it yields small but significant improvements on the more difficult datasets.

\section{Tensor Factorization of Knowledge Bases}
\begin{figure*}
\centering
\includegraphics[trim={2cm 11.75cm 2cm 11.4cm}, clip=true, width=0.98\textwidth]{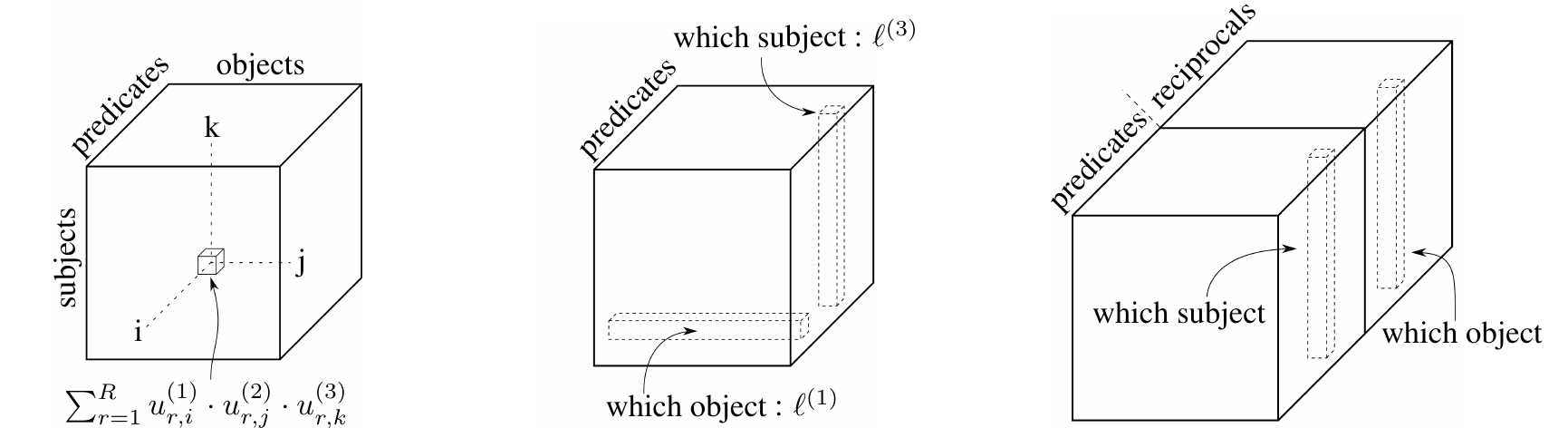}
\caption{(a) On the left, the link between the score of a triple (i,j,k) and the tensor estimated via \CP. (b) In the middle, the two type of fiber losses that we will consider. (c) On the right, our semantically invariant reformulation, the first-mode fibers become third-mode fibers of the reciprocal half of the tensor.}
\label{fig:tensors}
\end{figure*}
We describe in this section the formal framework we consider for knowledge base completion and more generally link prediction in relational data, the learning criteria, as well as the approaches that we will discuss.

\subsection{Link Prediction in Relational Data}
We consider relational data that comes in the form of triples (subject, predicate, object), where the subject and the object are from the same set of entities. In knowledge bases, these triples represent facts about entities of the world, such as $(W\!ashington, capital\_of, U\!S\!A)$. A training set $\indices$ contains triples of indices $\indices = \xSet{(i_1, j_1, k_1), ..., (i_\indecesz, j_\indecesz, k_\indecesz)}$ that represent predicates that are known to hold. The validation and test sets contain queries of the form $(?, j, k)$ and $(i, j, ?)$, created from triples $(i,j,k)$ that are known to hold but held-out from the training set. To give orders of magnitude, the largest datasets we experiment on, \fb and \yago, contain respectively $15k/1.3k$ and $123k/37$ entities/predicates.

\subsection{Tensor Decomposition for Link Prediction}

Relational data can be represented as a $\xSet{0,1}$-valued third order tensor $\tens{\tgt} \in \xSet{0,1}^{\nents\times\nrels\times\nents}$, where $\nents$ is the total number of entities and $\nrels$ the number of predicates, with $\tensi{\tgt}{i,j,k} = 1$ if the relation $(i,j,k)$ is known. In the rest of the paper, the three modes will be called the subject mode, the predicate mode and the object mode respectively. Tensor factorization algorithms can thus be used to infer a predicted tensor $\tens{\pred}\in\Re^{\nents\times\nrels\times\nents}$ that approximates $\tens{\tgt}$ in a sense that we describe in the next subsection. Validation/test queries $(?, j, k)$ are answered by ordering entities~$i'$ by decreasing values of $\tensi{\pred}{i',j,k}$, whereas queries $(i,j,?)$ are answered by ordering entities~$k'$ by decreasing values of $\tensi{\pred}{i,j,k'}$.

Several approaches have considered link prediction as a low-rank tensor decomposition problem. These models then differ only by structural constraints on the learned tensor. Three models of interest are:

\paragraph{\CP.} The canonical decomposition of tensors, also called CANDECOM/PARAFAC \citep{hitchcock_expression_1927}, represents a tensor $\tens{\unkn}\in\Re^{\nents_1\times\nents_2\times\nents_3}$ as a sum of $R$ rank one tensors $u^{(1)}_r \otimes u^{(2)}_r \otimes u^{(3)}_r$ (with $\otimes$ the tensor product) where $r\in\xSet{1,...,R}$, and $u^{(m)}_r\in\Re^{\nents_m}$:
\begin{equation*}
\tens{\unkn} = \sum_{r=1}^R u^{(1)}_r \otimes u^{(2)}_r \otimes u^{(3)}_r\,.
\end{equation*}
A representation of this decomposition, and the score of a specific triple is given in Figure~\ref{fig:tensors}~(a).
Given $\tens{\unkn}$, the smallest $R$ for which this decomposition holds is called the canonical rank of $\tens{\unkn}$. 

\paragraph{\distmult.} In the more specific context of link prediction, it has been suggested in \citet{bordes_learning_2011,nickel_three-way_2011} that since both subject and object mode represent the same entities, they should have the same factors. \distmult \citep{yang_embedding_2014} is a version of \CP with this additional constraint. It represents a tensor $\tens{\unkn}\in\Re^{\nents\times\nrels\times\nents}$ as a sum of rank-$1$ tensors $u^{(1)}_r \otimes u^{(2)}_r \otimes u^{(1)}_r$:
\begin{equation*}
\tens{\unkn} = \sum_{r=1}^R u^{(1)}_r \otimes u^{(2)}_r \otimes u^{(1)}_r\,.
\end{equation*}
\paragraph{\complex.} By contrast with the first models that proposed to share the subject and object mode factors, \distmult yields a tensor that is symmetric in the object and subject modes. The assumption that the data tensor can be properly approximated by a symmetric tensor for Knowledge base completion is not satisfied in many practical cases (e.g., while $(W\!ashington, capital\_of, U\!S\!A)$ holds, $(U\!S\!A, capital\_of, W\!ashington)$ does not). \complex \citep{trouillon_complex_2016} proposes an alternative where the subject and object modes share the parameters of the factors, but are complex conjugate of each other. More precisely, this approach represents a real-valued tensor $\tens{\unkn}\in\Re^{\nents_1\times\nents_2\times\nents_3}$ as the real part of a sum of $R$ complex-valued rank one tensors $u^{(1)}_r \otimes u^{(2)}_r \otimes \overline{u}^{(1)}_r$ where $r\in\xSet{1,...,R}$, and $u^{(m)}_r\in\Co^{\nents_m}$
\begin{equation*}
\tens{\unkn} = {\rm Re}\big(\sum_{r=1}^R u^{(1)}_r \otimes u^{(2)}_r \otimes \overline{u}^{(1)}_r\big)\,,
\end{equation*}
where $\overline{u}^{(1)}_r$ is the complex conjugate of $u^{(1)}_r$. This decomposition can represent any real tensor \citep{trouillon_complex_2016}.

The good performances of \distmult on notoriously non-symmetric datasets such as \fb or \wn are surprising. First, let us note that for the symmetricity to become an issue, one would have to evaluate queries $(i, j, ?)$ while also trying to answer correctly to queries of the form $(?, j, i)$ for a non-symmetric predicate $j$. The ranking for these two queries would be identical, and thus, we can expect issues with relations such as $capital\_of$. In \fb, those type of problematic queries make up only $4\%$ of the test set and thus, have a small impact. On \wn however, they make up $60\%$ of the test set. We describe in appendix 8.1 a simple strategy for \distmult to have a high filtered MRR on the hierarchical predicates of \wn despite its symmetricity assumption.

\subsection{Training}
Previous work suggested ranking losses \citep{bordes_translating_2013}, binary logistic regression \citep{trouillon_complex_2016} or sampled multiclass log-loss \citep{kadlec_knowledge_2017}. Motivated by the solid results in \citet{joulin2017fast}, our own experimental results, and with a satisfactory speed of about two minutes per epoch on \fb, we decided to use the full multiclass log-loss. 

Given a training triple $(i,j,k)$ and a predicted tensor $\tens{\unkn}$, the instantaneous multi-class log-loss $\loss_{i,j,k}(\tens{\unkn})$ is
\begin{align}
\label{eq:fiberloss}
\loss_{i,j,k}(\tens{\unkn}) & = \loss^{(1)}_{i,j,k}(\tens{\unkn}) + \loss^{(3)}_{i,j,k}(\tens{\unkn}) \\
\loss^{(1)}_{i,j,k}(\tens{\unkn}) &  =-\tensi{\unkn}{i,j,k} + \log \big ( \sum_{k'\!} \exp (\tensi{\unkn}{i,j,k'\!}) \big )\\
\loss^{(3)}_{i,j,k}(\tens{\unkn}) & = -\tensi{\unkn}{i,j,k} + \log \big ( \sum_{i'\!} \exp (\tensi{\unkn}{i',j,k\!}) \big )\,.
\end{align}
These two partial losses are represented in Figure~\ref{fig:tensors}~(b).
For \CP, the final tensor is computed by finding a minimizer of a regularized empirical risk formulation, where the factors $\facv{\dv}{\rkv}$ are weighted in a data-dependent manner by $\weightv{\dv}$, which we describe below: 
\begin{align}
\label{eq:mainprob}
\min_{(\facv{\dv}{\rkv})_{
\substack{\scriptscriptstyle \dv=1..3\\\scriptscriptstyle\rkv=1..\rk}}}
&\sum_{(i,j,k)\in\indices}
\loss_{i,j,k}\Big(\sum\limits_{\rkv=1}^\rk \!\facv{1}{\rkv} \otimes \facv{2}{\rkv}\otimes\facv{3}{\rkv}\Big) \nonumber \\
&+\lambda\sum_{\rkv=1}^\rk \sum_{\dv=1}^3 \norm{\weightv{\dv}\odot\facv{\dv}{\rkv}}_2^2\,, 
\end{align}
where $\odot$ is the entry-wise multiplication of vectors. 
For \distmult and \complex, the learning objective is similar, up to the appropriate parameter sharing and computation of the tensor.

As discussed in Section \ref{sec:matNorm}, the weights $\weightv{\dv}$ may improve performances when some rows/columns are sampled more than others. They appear naturally in optimization with stochastic gradient descent when the regularizer is applied only to the parameters that are involved in the computation of the instantaneous loss. For instance, in the case of the logistic loss with negative sampling used by \citet{trouillon_complex_2016}, denoting by $q^{\dv}_{i}$ the marginal probability (over $\indices$) that index $i$ appears in mode $\dv$ of a data triple, these weights are $\weightvi{\dv} = \sqrt{q^{\dv}_{i}+\alpha}$ for some $\alpha>0$ that depends on the negative sampling scheme.

We focus on redefining the loss \eqref{eq:fiberloss} and the regularizer \eqref{eq:mainprob}.

\section{Related Work}
\label{sec:rel_work}
We discuss here in more details the work that has been done on link prediction in relational data and on regularizers for tensor completion.

\subsection{Link Prediction in Relational Data}
There has been extensive research on link prediction in relational data, especially in knowledge bases, and we review here only the prior work that is most relevant to this paper. While some approaches explicitly use the graph structure during inference \citep{lao_random_2011}, we focus here on representation learning and tensor factorization methods, which are the state-of-the-art on the benchmark datasets we use. We also restrict the discussion to approaches that only use relational information, even though some approaches have been proposed to leverage additional types \citep{krompas_s_type-constrained_2015,ma2017transt} or external word embeddings \citep{toutanova_observed_2015}.

We can divide the first type of approaches into two broad categories. First, two-way approaches score a triple $(i,j,k)$ depending only on bigram interaction terms of the form subject-object, subject-predicate, and predicate-object. Even though they are tensor approximation algorithms of limited expressivity, two-way models based on translations \transe, or on bag-of-word representations \citep{joulin2017fast} have proved competitive on many benchmarks. Yet, methods using three-way multiplicative interactions, as described in the previous section, show the strongest performances \citep{bordes_learning_2011,garcia-duran_combining_2015,nickel_holographic_2015,trouillon_complex_2016}. Compared to general-purpose tensor factorization methods such as \CP, a common feature of these approaches is to share parameters between objects and subjects modes \citep{nickel_three-way_2011}, an idea that has been widely accepted except for the two-way model of \citet{joulin2017fast}. \distmult \citep{yang_embedding_2014} is the extreme case of this parameter sharing, in which the predicted tensor is symmetric in the subject and object modes. 

\subsection{Regularization for Matrix Completion}
\label{sec:matNorm}
Norm-based regularization has been extensively studied in the context
of matrix completion. The trace norm (or nuclear norm) has been
proposed as a convex relaxation of the rank
\citep{srebro_maximum-margin_2005} for matrix completion in the
setting of rating prediction, with strong theoretical guarantees
\citep{candes_exact_2009}. While efficient algorithms to solve the
convex problems have been proposed 
\citep[see e.g.][]{cai_singular_2010,jaggi_simple_2010}, the practice is still to
use the matrix equivalent of the nonconvex formulation
\eqref{eq:mainprob}. For the trace norm (nuclear $2$-norm), in the
matrix case, the regularizer simply becomes the squared $2$-norm of the factors and lends itself to alternating methods or SGD optimization
\citep{rennie_fast_2005,koren_matrix_2009}. When the samples are not
taken uniformly at random from a matrix, some other norms are
preferable to the usual nuclear norm. The weighted trace norm
reweights elements of the factors based on the marginal rows and columns sampling probabilities, which can improve sample complexity bounds when sampling is non-uniform  \citep{foygel_learning_2011,negahban_restricted_2012}. Direct SGD
implementations on the nonconvex formulation implicitly take this reweighting rule into account and were used by the winners of the Netflix challenge~\citep[see][Section 5]{srebro_collaborative_2010}.

\subsection{Tensor Completion and Decompositions}
There is a large body of literature on low-rank tensor decompositions
\citep[see][for a comprehensive review]{kolda_tensor_2009}. Closely
related to our work is the canonical decomposition of tensor (also
called CANDECOMP/PARAFAC or \CP) \citep{hitchcock_expression_1927},
which solves a problem similar to \eqref{eq:nonconvex} without the
regularization (i.e., $\lambda=0$), and usually the square
loss.

Several norm-based regularizations for tensors have been proposed. Some are based on unfolding a tensor along each of its modes to obtain matricizations, and either regularize by the sum of trace norms of the matricizations~\citep{tomioka_estimation_2010} or write the original tensor
as a sum of tensors $T_k$, regularizing their respective $k$th matricizations with the trace norm ~\citep{wimalawarne_multitask_2014}. However, in the large-scale setting, even rank-1 approximations of matricizations involve too many parameters to be tractable.

Recently, the tensor trace norm (nuclear $2$-norm) was proposed as a regularizer for tensor completion \citet{yuan_tensor_2016}, and an algorithm based on the generalized conditional gradient has been developed by \citet{cheng_scalable_2016}. This algorithm requires, in an inner loop, to compute a (constrained) rank-1 tensor that has largest dot-product with the gradient of the data-fitting term (gradient w.r.t.\ the tensor argument). This algorithm is efficient in our setup only with the square error loss (instead of the multiclass log-loss), because the gradient is then a low-rank + sparse tensor when the argument is low-rank. However, on large-scale knowledge bases, the state of the art is to use a binary log-loss or a multiclass log-loss
\citep{trouillon_complex_2016,kadlec_knowledge_2017}; in that case, the gradient is not adequately structured, thereby causing the approach of \cite{cheng_scalable_2016} to be too computationally costly.

\section{Nuclear $p$-Norm Regularization}
As discussed in Section~\ref{sec:rel_work}, norm-based regularizers have proved useful for matrices. We aim to reproduce these successes with tensor norms. We use the nuclear $p$-norms defined by \citet{friedland_nuclear_2014}. As shown in Equation~\eqref{eq:mainprob}, the community has favored so far a regularizer based on the square Frobenius norms of the factors \citep{yang_embedding_2014, trouillon_complex_2016}. We  first show that the unweighted version of this regularizer is not a tensor norm. Then, we propose4 a variational form of the nuclear $3$-norm to replace the usual regularization at no additional computational cost when used with SGD. Finally, we discuss a weighting scheme analogous to the weighted trace-norm proposed in \citet{srebro_collaborative_2010}.

\subsection{From Matrix Trace-Norm to Tensor Nuclear Norms}
To simplify notation, let us introduce the set of \CP decompositions of a tensor $\tens{\unkn}$ of rank at most $R$:
\begin{align}
\allfaccp=\Big \{(\facv{\dv}{\rkv})_{\substack{\dv=1..3\\\rkv=1..\rk}}~\Big|~ &\tens{\unkn} = \sum_{\rkv=1}^\rk \!\facv{1}{\rkv} \otimes \facv{2}{\rkv}\otimes\facv{3}{\rkv},\\
&\forall r, d, \: \facv{\dv}{\rkv} \in\Re^{N_d}\Big \}.
\end{align}
We will study the family of regularizers:
\begin{equation}
\label{eq:omegas}
\Omega^{\alpha}_p(u) = \frac{1}{3}\sum_{r=1}^R\sum_{d=1}^3\norm{u_r^{(d)}}_p^{\alpha}.
\end{equation}
Note that with $p=\alpha=2$, we recover the familiar squared Frobenius norm regularizer used in \eqref{eq:mainprob}. Similar to showing that the squared Frobenius norm is a \emph{variational form} of the trace norm on matrices (i.e., its minimizers realize the trace norm, $\inf_{M=UV^T}\frac{1}{2}(\lVert U\lVert_F^2+\lVert V\lVert_F^2) = \lVert M\lVert_*$), we start with a technical lemma that links our regularizer with a function on the spectrum of our decompositions. 
\begin{lemma}
\begin{equation}
\label{eq:lemma_tech}
\min_{u\in \allfaccp}\frac{1}{3}\sum_{r=1}^R\sum_{d=1}^3\norm{u_r^{(d)}}_p^{\alpha} = \min_{u \in \allfaccp}\sum_{r=1}^R\prod_{d=1}^3\norm{u_r^{(d)}}_p^{\alpha/3}.
\end{equation}
Moreover, the minimizers of the left-hand side satisfy: 
\begin{equation}
\norm{u_r^{(d)}}_p = \sqrt[3]{\prod_{d'=1}^3\norm{u_r^{(d')}}_p}.
\end{equation}
\end{lemma}
\begin{proof}
See Appendix~8.2.
\end{proof}
This Lemma motivates the introduction of the set of $p$-norm normalized tensor decompositions:
\begin{align}
\nallfaccp{p}=\Big \{&(\sigma_{\rkv}, (\tilde{u}_r))_{r=1..R}~\Big|~ \sigma_{\rkv} = \prod_{d=1}^3\norm{u^{(d)}_{\rkv}}_p,\nonumber \\
&\tilde{u}^{(d)}_r = \frac{u^{(d)}_r}{\norm{u^{(d)}_r}_p}, \forall r,d, u \in\allfaccp \Big \}.
\end{align}
Lemma~\ref{eq:lemma_tech}, shows that $\Omega_p^\alpha$ behaves as an $\ell_{\alpha/D}$ penalty over the CP \emph{spectrum} for tensors of order $D$. We recover the nuclear norm for matrices when $\alpha = p = 2$.

Using Lemma~\ref{eq:lemma_tech}, we have :
\begin{equation}
\label{eq:rescale}
\min_{u\in\allfaccp}\Omega_2^2(u) \leq \eta \iff \min_{(\sigma, \tilde{u}) \in \nallfaccp{2}}\lVert \sigma\lVert_{2/3} \leq \eta^{3/2}
\end{equation}

We show that the sub-level sets of the term on the right are not convex, which implies that $\Omega_2^2$ is not the variational form of a tensor norm, and hence, is not the tensor analog to the matrix trace norm.

\begin{proposition}
\label{prop:notanorm}
The function over third order-tensors of  $\Re^{\nents_1 \times \nents_2\times\nents_3}$ defined as
$$\opnorm{\tens{\unkn}} = \min\Big\{ \norm{\sigma}_{2/3}~\Big|~(\sigma, \tilde{u}) \in \nallfaccp{2}, \rk\in\Na\Big\}$$
is not convex.
\end{proposition}
\begin{proof}
See Appendix~8.2.
\end{proof}
\begin{remark}
\citet[Appendix 
D]{cheng_scalable_2016} already showed that regularizing with the square Frobenius norm of the factors is not related to the trace norm for tensors of order $3$ and above, but their observation is that the regularizer is not positively homogeneous, i.e., $\min_{u\in\alpha\allfaccp}\Omega_2^2(u) \neq |\alpha|\min_{u\in\allfaccp}\Omega_2^2(u)$. Our result in Proposition \ref{prop:notanorm} is stronger in that we show that this regularizer is not a norm even after the rescaling \eqref{eq:rescale} to make it homogeneous.
\end{remark}
The nuclear $p$-norm of $\tens{\unkn}\in\Re^{\nents_1\times\nents_2\times\nents_3}$ for $p\in[1, +\infty]$, is defined in \citet{friedland_nuclear_2014} as
\begin{align}
\npnorm{\tens{\unkn}} :=
\min
\Big\{
\norm{\sigma}_1 ~\Big|~ (\sigma, \tilde{u}) \in \nallfaccp{p}, \rk\in\Na
\Big\}\,. \label{eq:npnorm}
\end{align}
Given an estimated upper bound on the optimal $\rk$, the original problem \eqref{eq:mainprob} can then be re-written as a non-convex problem using the equivalence in Lemma~\ref{eq:lemma_tech}:
\begin{align}
\min_{(\facv{\dv}{\rkv})_{
\substack{\scriptscriptstyle \dv=1..3\\\scriptscriptstyle\rkv=1..\rk}}}&
\sum_{(i,j,k)\in\indices}
\loss_{i,j,k}\Big(\sum\limits_{\rkv=1}^\rk \!\facv{1}{\rkv} \otimes \facv{2}{\rkv}\otimes\facv{3}{\rkv}\Big) \\
+ &\frac{\lambda}{3}\sum_{\rkv=1}^\rk \sum_{\dv=1}^3 \norm{\facv{\dv}{\rkv}}_p^3\,. \label{eq:nonconvex}
\end{align}

This variational form suggests to use $p=3$, as a means to make the regularizer separable in each coefficients, given that then $\norm{\facv{\dv}{\rkv}}_p^3= \sum_{i=1}^{\nv{\dv}} \big|\facv{\dv}{\rkv, i}|^3$. 

\subsection{Weighted Nuclear $p$-Norm}

Similar to the weighted trace-norm for matrices, the weighted nuclear $3$-norm can be easily implemented by keeping the regularization terms corresponding to the sampled triplets only, as discussed in Section \ref{sec:matNorm}. This leads to a formulation of the form 
\begin{align}
\min_{(\facv{\dv}{\rkv})_{
\substack{\scriptscriptstyle \dv=1..3\\\scriptscriptstyle\rkv=1..\rk}}} \sum_{(i,j,k)\in\indices}& \Big [ \loss_{i,j,k}\big(\sum\limits_{\rkv=1}^\rk \!\facv{1}{\rkv} \!\otimes\! \facv{2}{\rkv}\!\otimes\!\facv{3}{\rkv}\big) \label{eq:weighted_nuclear_three} \\
+ &\frac{\lambda}{3}\sum_{r=1}^R\Big (\big|\facv{1}{\rkv, i}|^3 +  \big|\facv{2}{\rkv, j}|^3 +\big|\facv{3}{\rkv, k}|^3 \Big) \Big ]. \nonumber
\end{align}
For an example $(i, j, k)$, only the parameters involved in the computation of $\pred_{i,j,k}$ are regularized. The computational complexity is thus the same as the currently used weighted Frobenius norm regularizer. With $q^{(1)}$ (resp. $q^{(2)}$, $q^{(3)}$) the marginal probabilities of sampling a subject (resp. predicate, object), the weighting implied by this regularization scheme is 
$$\norm{\tens{X}}_{*,3,w} = \norm{(\sqrt[3]{q^{(1)}}\otimes\sqrt[3]{q^{(2)}}\otimes\sqrt[3]{q^{(3)}})\odot\tens{X}}_{*,3}$$
We justify this weighting only by analogy with the matrix case discussed by \citep{srebro_collaborative_2010}: to make the weighted nuclear $3$-norm of the all $1$ tensor independent of its dimensions for a uniform sampling (since the nuclear $3$-norm grows as $\sqrt[3]{MNP}$ for an $(M, N, P)$ tensor).

Comparatively, for the weighted version of the nuclear $2$-norm analyzed in \citet{yuan_tensor_2016}, the nuclear $2$-norm of the all $1$ tensor scales like $\sqrt{NMP}$. This would imply a formulation of the form
\begin{align}
\min_{(\facv{\dv}{\rkv})_{
\substack{\scriptscriptstyle \dv=1..3\\\scriptscriptstyle\rkv=1..\rk}}}
&\sum_{(i,j,k)\in\indices}
\loss_{i,j,k}\Big(\sum\limits_{\rkv=1}^\rk \!\facv{1}{\rkv} \otimes \facv{2}{\rkv}\otimes\facv{3}{\rkv}\Big) \nonumber \\
&+ \frac{\lambda}{3}\sum_{\rkv=1}^\rk \sum_{\dv=1}^3 \norm{\sqrt{q^{(d)}} \odot \facv{\dv}{\rkv}}_2^3\,.\label{eq:weighted_nuclear_two}
\end{align}

Contrary to formulation~\eqref{eq:weighted_nuclear_three}, the optimization of formulation~\eqref{eq:weighted_nuclear_two} with a minibatch SGD leads to an update of every coefficients for each mini-batch considered. Depending on the implementation, and size of the factors, there might be a large difference in speed between the updates of the weighted nuclear $\{2,3\}$-norm. In our implementation, this difference for \CP is of about $1.5\times$ in favor of the nuclear $3$-norm on \fb.

\section{A New \CP Objective}
Since our evaluation objective is to rank either the left-hand side or right-hand side of the predicates in our dataset, what we are trying to achieve is to model both predicates and their reciprocal. This suggests appending to our input the reciprocals of each predicates, thus factorizing $[\tens{Y};_2 \tilde{\tens{Y}}]$ rather than $\tens{Y}$, where $[~;_2~]$ is the mode-2 concatenation, and $\tens{Y}_{i,j,k} = \tilde{\tens{Y}}_{k,j,i}$. After that, we only need to model the object fibers of this new tensor $\tens{Y}$. We represent this transformation in Figure~\ref{fig:tensors}~(c). This reformulation has an important side-effect: it makes our algorithm invariant to the arbitrary choice of including a predicate or its reciprocal in the dataset. This property was introduced as "Semantic Invariance" in \citet{bailly_semantically_2015}. 
Another way of achieving this invariance property would be to find the flipping of predicates that lead to the smallest model. In the case of a CP decomposition, we would try to find the flipping that leads to lowest tensor rank. This seems hopeless, given the NP-hardness of computing the tensor rank.

More precisely, the instantaneous loss of a training triple $(i,j,k)$ becomes :
\begin{align}
\loss_{i,j,k}(\tens{\unkn}) = &-
\tensi{\unkn}{i,j,k} + \log \big ( \sum_{k'\!} \exp (\tensi{\unkn}{i,j,k'\!}) \big ) \label{eq:inv_loss} \\
& -\tensi{\unkn}{k,j+P,i} + \log \big ( \sum_{i'\!} \exp (\tensi{\unkn}{k,j+P,i'\!}) \big ). \nonumber
\end{align}
At test time we use $\tensi{\pred}{i,j,:}$ to rank possible right hand sides for query $(i, j, ?)$ and $\tensi{\pred}{k,j+P,:}$ to rank possible left hand sides for query $(?, j, k)$.

Using \CP to factor the tensor described in~\eqref{eq:inv_loss}, we beat the previous state of the art on many benchmarks, as shown in Table~\ref{tab:res}. This reformulation seems to help even the \complex decomposition, for which parameters are shared between the entity embeddings of the first and third mode.

\begin{table}[t]
\centering
\begin{tabular}{cccccc}
\toprule
Dataset & N & P & Train & Valid & Test \\
\wn & $41$k & $18$ & $141$k & $5$k & $5$k \\
\wnrr & $41$k & $11$ & $87$k & $3$k & $3$k \\
\fb & $15$k & 1k & $500$k & $50$k & $60$k \\
\fbd & $15$k & $237$ & $272$k & $18$k & $20$k \\
\yago & $123$k & $37$ & $1$M & $5$k & $5$k \\
\bottomrule
\end{tabular}
\caption{Dataset statistics.}
\label{tab:stats}
\end{table}

\section{Experiments}
We conducted all experiments on a Quadro GP 100 GPU. The code is available at \url{https://github.com/facebookresearch/kbc}. 

\begin{table*}[t]
\centering
\begin{tabular}{clcccccccccc}
\toprule

{} & Model &  \multicolumn{2}{c}{\wn} & \multicolumn{2}{c}{\wnrr} & \multicolumn{2}{c}{\fb} & \multicolumn{2}{c}{\fbd} & \multicolumn{2}{c}{\yago}\\
\midrule
{} & {}          &{\small  MRR}  &{\small H@10}  &{\small  MRR}  &{\small H@10}  &{\small  MRR}  &{\small H@10}  &{\small  MRR}  &{\small H@10}  &{\small  MRR}  &{\small H@10} \\

\multirow{4}{*}{\rotatebox[origin=c]{90}{Past SOTA}} & \CP
                      & $0.08$ & $0.13$ &  -  &  -  & $0.33$ & $0.53$   &  -  &  -  &   -    &   - \\
{} & \complex${}^\dagger$         & $0.94$ & $0.95$ &  $0.44$ & $0.51$  & $0.70$ & $0.84$ & $0.25$ & $0.43$ &  $0.36$ & $0.55$ \\
{} & \distmult$^\ast$  & $0.82$ & $0.94$ &  $0.43$ & $0.49$  & $0.80$ & $0.89$ & $0.24$ & $0.42$ &  $0.34$ & $0.54$ \\
{} & \conve$^\ast$     & $0.94$ & $0.95$ &  $0.46$ & $0.48$  & $0.75$ & $0.87$ & $0.32$ & $0.49$ &  $0.52$ & $0.66$ \\
{} & Best Published$^\star$    & $0.94$ & $\bm{0.97}$ & $0.46$ & $0.51$ & $0.84$ & $\bm{0.93}$ & $0.32$ & $0.49$ & $0.52$ & $0.66$ \\

\midrule

\multirow{4}{*}{\rotatebox[origin=c]{90}{Standard}} &&&&&&&&&&&\\
{} & \CP-N3        & $0.20$      & $0.33$      &  $0.12$ & $0.20$  & $0.46$ & $0.65$ & $0.33$ & $0.51$ &  $0.38$ & $0.65$ \\
{} & \complex-N3   & $\bm{0.95}$ & $\bm{0.96}$ &  $0.47$ & $0.54$  & $0.80$ & $0.89$ & $0.35$ & $0.54$ &  $0.49$ & $0.68$ \\
\\
\midrule

\multirow{4}{*}{\rotatebox[origin=c]{90}{Reciprocal}} & \CP-FRO
                  & $\bm{0.95}$  & $0.95$      &  $0.46$      & $0.48$      & $\bm{0.86}$  & $0.91$ & $0.34$       & $0.51$       &   $0.54$    & $0.68$ \\
{} & \CP-N3       & $\bm{0.95}$  & $0.96$ &  $0.47$      & $0.54$      & $\bm{0.86}$  & $0.91$ & $0.36$       & $0.54$       &   $0.57$    & $\bm{0.71}$ \\
{} & \complex-FRO & $\bm{0.95}$  & $0.96$ &  $0.47$      & $0.54$      & $\bm{0.86}$  & $0.91$ & $0.35$       & $0.53$       &   $0.57$    & $\bm{0.71}$\\
{} & \complex-N3  & $\bm{0.95}$  & $0.96$ &  $\bm{0.48}$ & $\bm{0.57}$ & $\bm{0.86}$  & $0.91$ & $\bm{0.37}$  &  $\bm{0.56}$ & $\bm{0.58}$ & $\bm{0.71}$ \\
\bottomrule
\end{tabular}
\caption{${}^\ast$Results taken as best from \citet{dettmers2017convolutional} and \citet{kadlec_knowledge_2017}. ${}^\dagger$Results taken as best from \citet{dettmers2017convolutional} and \citet{trouillon_complex_2016}.${}^\star$ We give the origin of each result on the Best Published row in appendix.}
\label{tab:res}
\end{table*}

\subsection{Datasets and Experimental Setup}
\wn and \fb are popular benchmarks in the Knowledge Base Completion community. The former comes from the WordNet database, was introduced in \citet{bordes_semantic_2014} and describes relations between words. The most frequent types of relations are highly hierarchical (e.g., hypernym, hyponym). The latter is a subsampling of Freebase limited to $15$k entities, introduced in \citet{bordes_translating_2013}. It contains predicates with different characteristics (e.g., one-to-one relations such as \emph{capital\_of} to many-to-many such as \emph{actor\_in\_film}).

\citet{toutanova_observed_2015} and \citet{dettmers2017convolutional} identified train to test leakage in both these datasets, in the form of test triplets, present in the train set for the reciprocal predicates. Thus, both of these authors created two modified datasets: \fbd and \wnrr. These datasets are harder to fit, so we expect regularization to have more impact. \citet{dettmers2017convolutional} also introduced the dataset \yago, which is larger in scale and doesn't suffer from leakage. All datasets statistics are shown in Table~\ref{tab:stats}.

In all our experiments, we distinguish two settings: Reciprocal, in which we use the loss described in equation~\eqref{eq:inv_loss} and Standard, which uses the loss in equation~\eqref{eq:fiberloss}. We compare our implementation of \CP and \complex with the best published results, then the different performances between the two settings, and finally, the contribution of the regularizer in the reciprocal setting. In the Reciprocal setting, we compare the weighted nuclear $3$-norm (N3) against the regularizer described in~\eqref{eq:mainprob} (FRO). In preliminary experiments, the weighted nuclear $2$-norm described in ~\eqref{eq:weighted_nuclear_two} did not seem to perform better than N3 and was slightly slower. We used Adagrad \citep{duchi_adaptive_2011} as our optimizer, whereas \citet{kadlec_knowledge_2017} favored Adam \citep{kingma2014adam}, because preliminary experiments didn't show improvements.

We ran the same grid for all algorithms and regularizers on the \fb, \fbd, \wn, \wnrr datasets, with a rank set to $2000$ for \complex, and $4000$ for \CP. Our grid consisted of two learning rates: $10^{-1}$ and $10^{-2}$, two batch-sizes: $25$ and $100$, and regularization coefficients in $[0, 10^{-3}, 5.10^{-3}, 10^{-2}, 5.10^{-2}, 10^{-1}, 5.10^{-1}]$. On \yago, we limited our models to rank $1000$ and used batch-sizes $500$ and $3000$, the rest of the grid was identical. We used the train/valid/test splits provided with these datasets and measured the filtered  Mean Reciprocal Rank (MRR) and Hits@10 (\citet{bordes_translating_2013}). We used the filtered MRR on the validation set for early stopping and report the corresponding test metrics. In this setting, an epoch for \complex with batch-size 100 on \fb took about $110s$ and $325s$ for a batch-size of $25$. We trained for $100$ epochs to ensure convergence, reported performances were reached within the first $25$ epochs.

All our results are reported in Table~\ref{tab:res} and will be discussed in the next subsections. Besides our implementations of \CP and \complex, we include the results of \conve and \distmult in the baselines. The former because \citet{dettmers2017convolutional} includes performances on the \wnrr and \yago benchmarks, the latter because of the good performances on \fb of \distmult and the extensive experiments on \wn and \fb reported in \citet{kadlec_knowledge_2017}. The performances of \distmult on \fbd, \wnrr and \yago may be slightly underestimated, since our baseline \CP results are better. To avoid clutter, we did not include in our table of results algorithms that make use of external data such as types \citep{krompas_s_type-constrained_2015},  external word embeddings \citep{toutanova_observed_2015}, or using path queries as regularizers \citep{guu_traversing_2015}. The published results corresponding to these methods are subsumed in the "Best Published" line of Table~\ref{tab:res}, which is taken, for every single metric and dataset, as the best published result we were able to find.

\subsection{Reimplementation of the Baselines}
The performances of our reimplementation of \CP and \complex appear in the middle rows of Table~\ref{tab:res} (Standard setting). We only kept the results for the nuclear $3$-norm, which didn't seem to differ from those with the Frobenius norm. Our results are slightly better than their published counterparts, going from $0.33$ to $0.46$ filtered MRR on \fb for \CP and $0.70$ to $0.80$ for \complex.  This might be explained in part by the fact that in the Standard setting \eqref{eq:mainprob} we use a multi-class log-loss, whereas \citet{trouillon_complex_2016} used binomial negative log-likelihood. Another reason for this increase can be the large rank of $2000$ that we chose, where previously published results used a rank of around $200$; the more extensive search for optimization/regularization parameters and the use of nuclear $3$-norm instead of the usual regularization are also most likely part of the explanation.

\subsection{Standard vs Reciprocal}
In this section, we compare the effect of reformulation~\eqref{eq:inv_loss}, that is, the middle and bottom rows of Table~\ref{tab:res}. The largest differences are obtained for \CP, which becomes a state of the art contender going from $0.2$ to $0.95$ filtered MRR on \wn, or from $0.46$ to $0.86$ filtered MRR on \fb.For \complex, we notice a weaker, but consistent improvement by using our reformulation, with the biggest improvements observed on \fb and \yago. 
Following the analysis in \citet{bordes_translating_2013}, we show in Table~\ref{tab:mrr_stats} the average filtered MRR as a function of the degree of the predicates. We compute the average in and out degrees on the training set, and separate the predicates in $4$ categories : 1-1, 1-m, m-1 and m-m, with a cut-off at $1.5$ on the average degree. We include reciprocal predicates in these statistics. That is, a predicate with an average in-degree of $1.2$ and average out-degree of $3.2$ will count as a 1-m when we predict its right-hand side, and as an m-1 when we predict its left-hand side. Most of our improvements come from the 1-m and m-m categories, both on \complex and \CP.
\begin{table}[t]
\centering
\begin{tabular}{lcccc}
\toprule
{} & \bf{\texttt{1-1}} & \bf{\texttt{m-1}} & \bf{\texttt{1-m}} & \bf{\texttt{m-m}} \\
\CP Standard          & $0.45$ & $0.71$ & $0.24$ & $0.44$\\
\CP Reciprocal        & $0.77$ & $\bm{0.92}$ & $\bm{0.71}$ & $0.86$\\
\complex Standard     & $0.87$ & $\bm{0.92}$ & $0.59$ & $0.81$ \\
\complex Reciprocal & $\bm{0.88}$ & $\bm{0.92}$ & $\bm{0.71}$ & $\bm{0.87}$\\
\bottomrule
\end{tabular}
\caption{Average MRR per relation type on \fb.}
\label{tab:mrr_stats}
\end{table}

\subsection{Frobenius vs Nuclear $3$}
We focus now on the effect of our norm-based N3 regularizer, compared to the Frobenius norm regularizer favored by the community. Comparing the four last rows of Table~\ref{tab:res}, we notice a small but consistent performance gain across datasets. The biggest improvements appear on the harder datasets \wnrr, \fbd and \yago. We checked on \wnrr the significance of that gain with a Signed Rank test on the rank pairs for \CP.

\subsection{Effect of Optimization Parameters}

During these experiments, we noticed a heavy influence of optimization hyper-parameters on final results. This influence can account for as much as $0.1$ filtered MRR and is illustrated in Figure~\ref{fig:optimization}.

\section{Conclusion and Discussion}
The main contribution of this paper is to isolate and systematically explore the effect of different factors for large-scale knowledge base completion. While the impact of optimization parameters was well known already, neither the effect of the formulation (adding reciprocals doubles the mean reciprocal rank on \fb for \CP) nor the impact of the regularization was properly assessed. The conclusion is that the \CP model performs nearly as well as the competitors when each model is evaluated in its optimal configuration. We believe this observation is important to assess and prioritize directions for further research on the topic. 

In addition, our proposal to use nuclear $p$-norm as regularizers with $p\neq 2$ for tensor factorization in general is of independent interest. 

The results we present leave several questions open. Notably, whereas we give definite evidence that \CP itself can perform extremely well on these datasets as long as the problem is formulated correctly, we do not have a strong theoretical justification as to why the differences in performances are so significant.

\begin{figure}
\centering
\includegraphics[scale=0.38]{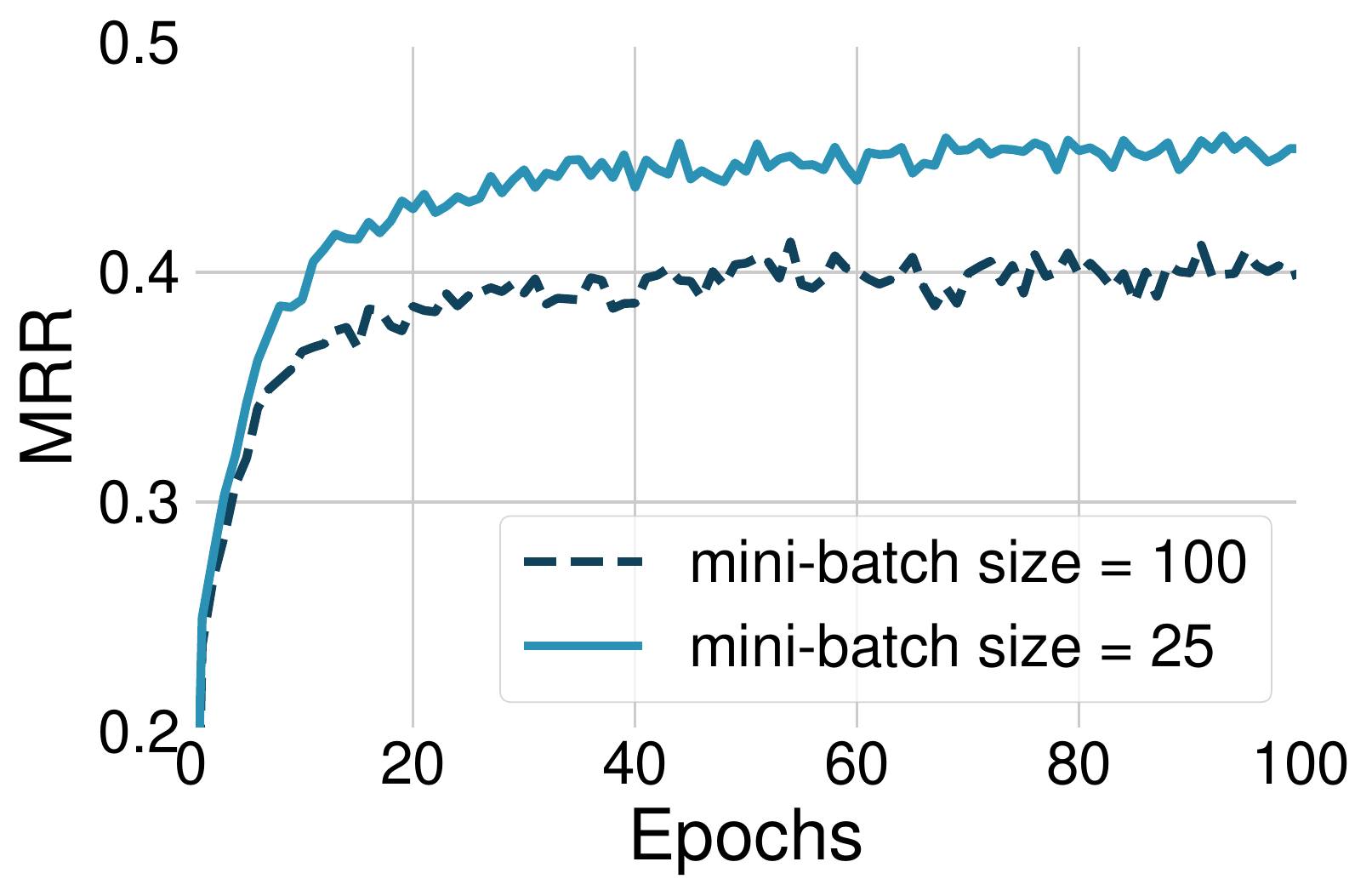}
~
\includegraphics[scale=0.38]{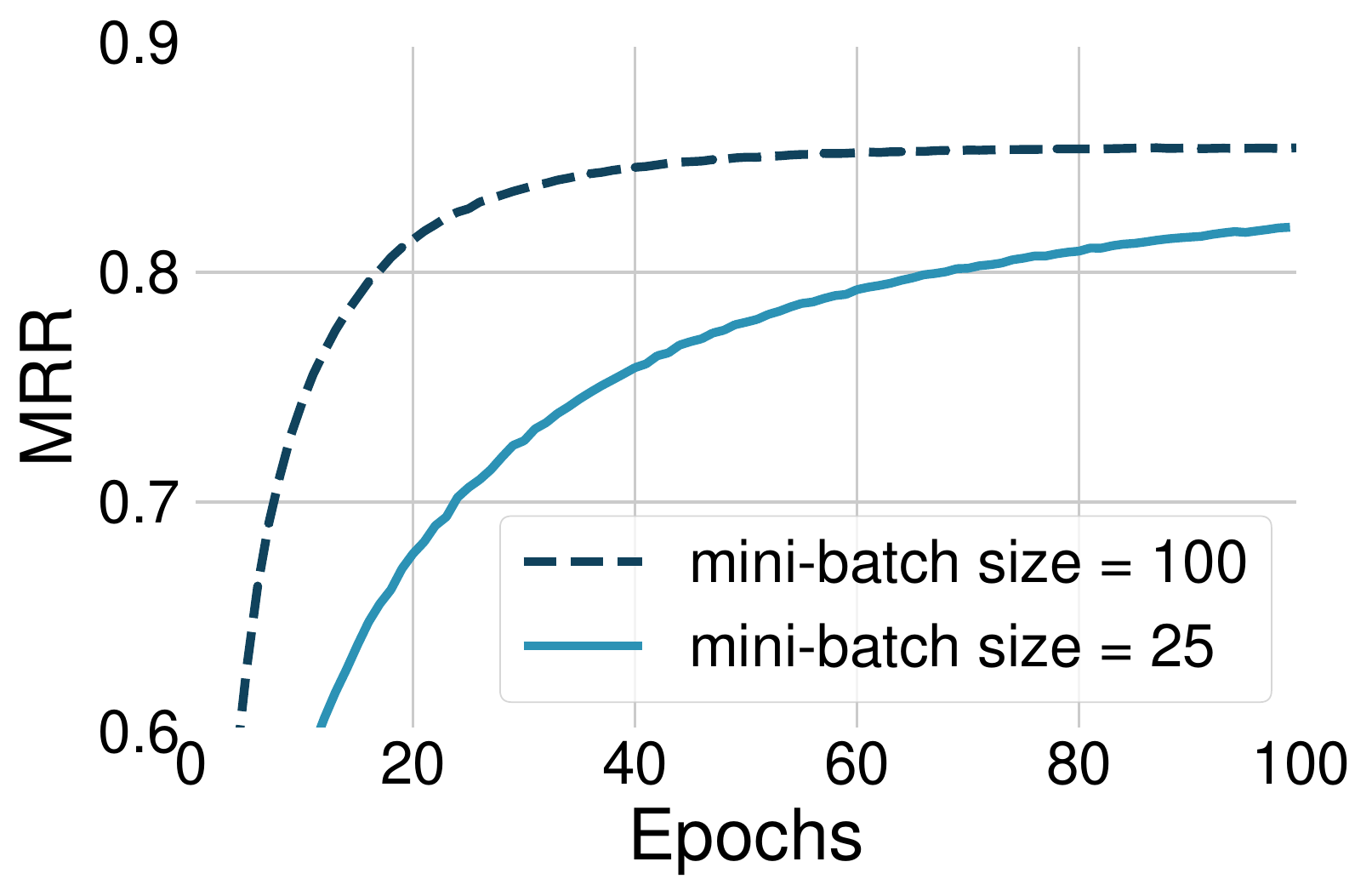}
\caption{Effect of the batch-size on \fb in the Standard (top) and Reciprocal (bottom) settings, other parameters being equal. The difference is large even after $100$ epochs and the effect is inverted in the two settings, making it hard to choose the batch-size a priori.
}
\label{fig:optimization}
\end{figure}

\section*{Acknowledgements}
The authors thank Armand Joulin and Maximilian Nickel for valuable discussions.

\bibliography{paper}
\bibliographystyle{icml2018}

\clearpage

\section{Appendix}
\subsection{DistMult on hierarchical predicates}
Suppose we are trying to embed a single hierarchical predicate $p$ which is a $n$-ary tree of depth $d$. This tree will have $n^d$ leaves, $1$ root and $K_n^d=\frac{n^d-n}{n-1}\sim n^{d-1}$ internal nodes. We forget any modeling issues we might have, but focus on the symmetricity assumption in {\tt Distmult}.

Leaves and the root only appear on one side of the queries $(i, p, j)$ and hence won't have any problems with the symmetricity. We now focus on an internal node $i$. It has $n$ children $(c^i_k)_{k=1..n}$ and one ancestor $a_i$. Assuming $n > 2$, the MRR associated with this node will be higher if the query $(i, p, ?)$ yields the ranked list $[c^i_1, ..., c^i_n, a_i]$. Indeed, the filtered rank of the n queries $(i, p, c^i_k)$ will be $1$ while the filtered rank of the query $(a_i,p, i)$ will be $n+1$.

Counting the number of queries for which the filtered rank is $1$, we see that they far outweigh the queries for which the filtered rank is $n+1$ in the final filtered MRR. For each internal nodes, $n$ queries lead to a rank of $1$, and only $1$ to a rank of $n+1$. For the root, $n$ queries with a rank of $1$, for the leaves, $n^d$ queries with a rank of $1$.

Our final filtered MRR is :
\begin{align}
mrr &= \frac{n^d + n +K_n^d\frac{1}{n+1}+K_n^d n}{n^d + n + (n+1)K_n^d} \\
&= 1 - \underbrace{K_n^d\frac{n/(n+1)}{n^d + n + (n+1)K_n^d}}_{\sim\frac{1}{2n}} \rightarrow 1
\end{align}

Hence for big hierarchies such as hyponym or hypernym in {\tt WN}, we expect the filtered MRR of {\tt DistMult} to be high even though its modeling assumptions are incorrect.

\subsection{Proofs}
\label{app:proof}
\begin{lemma}
\begin{equation}
\label{eq:lemma_tech}
\min_{u\in \allfaccp}\frac{1}{3}\sum_{r=1}^R\sum_{d=1}^3\norm{u_r^{(d)}}_p^{\alpha} = \min_{u \in \allfaccp}\sum_{r=1}^R\sqrt[3]{\prod_{d=1}^3\norm{u_r^{(d)}}_p^{\alpha}}.
\end{equation}
Moreover, the minimizers of the left-hand side satisfy: 
\begin{equation}
\norm{u_r^{(d)}}_p = \sqrt[3]{\prod_{d'=1}^3\norm{u_r^{(d')}}_p}.
\end{equation}
\end{lemma}
\begin{proof}
First, we characterize the minima :
\begin{align}
&\min\Big\{\Omega_p^\alpha(u)~\Big|~u\in\allfaccp\Big\} \\
=&\min\Big\{\Omega_p^\alpha(\tilde{u})~\Big|~\tilde{u}^d_r = c^d_r u^d_r, u\in\allfaccp, \prod_{d=1}^3 c^d_r=1\Big\} \\
=&\min\Big\{\frac{1}{3}\sum_{r=1}^R\sum_{d=1}^3(|c_r^d|\norm{u_r^{(d)}}_p)^\alpha~\Big|\\
&\quad\qquad\tilde{u}^d_r = c^d_r u^d_r, u\in\allfaccp, \prod_{d=1}^3 c^d_r=1\Big\} \\
\end{align}
We study a summand, for $c_i, a_i > 0$ :
$$\min_{\prod_{d=1}^3 c^d=1}\frac{1}{3}\sum_{d=1}^3(c_i a_i)^\alpha$$
Using constrained optimization techniques, we obtain that this minimum is obtained for :
$$c_i = \frac{\sqrt[3]{\prod_{d=1}^3a_i}}{a_i}$$
and has value $(\prod_{d=1}^3a_i)^{\alpha/3}$, which completes the proof.
\end{proof}

\begin{proposition}
The function over third order-tensors of  $\Re^{\nents_1 \times \nents_2\times\nents_3}$ defined as
$$\opnorm{\tens{\unkn}} = \min\Big\{ \norm{\sigma}_{2/3}~\Big|~(\sigma, \tilde{u}) \in \allfaccp, \rk\in\Na\Big\}$$
is not convex.
\end{proposition}
\begin{proof}
We first study elements of $\Re^{2\times 2 \times 1}$, tensors of order $3$ associated with matrices of size $2\times 2$. We have that 
\begin{equation}
\opnorm{\begin{pmatrix}
1 & 0 \\
0 & 0
\end{pmatrix}} = \opnorm{\begin{pmatrix}
0 & 0 \\
0 & 1
\end{pmatrix}}= 1
\end{equation}
Let $A=\frac{1}{2}I_2$, the mean of these two matrices. Identifying $A$ with a $2\times 2\times 1$ tensor $\tens{A}$ to obtain the decomposition $(\sigma, u)$ yielding $\opnorm{\tens{A}}$, we have that the matrix $A$ can be written as $\sum_{r=1}^R \sigma_r u^{(1)}_r\otimes u^{(2)}_r$. This comes from the fact that $u^{(3)}_r$ is a normalized $1\times 1$ vector, so its only entry is equal to $1$. We then write that trace $\Tr(A) = \sum_{r=1}^R \sigma_r \Tr( u^{(1)}_r\otimes u^{(2)}_r) \leq \sum_{r=1}^R \sigma_r$ by Cauchy-Schwarz. Hence $\norm{\sigma}_1 \geq \Tr(A) = 1$. Moreover, we have $\norm{\sigma}_{2/3}\geq\norm{\sigma}_1$ with equality only for $\sigma$ with at most one non-zero coordinate. Since $A$ is of rank $2$, its representation has at least $2$ non-zero coordinates, hence $\opnorm{A} = \norm{\sigma}_{2/3} > 1$, which contradicts convexity. This proof can naturally be extended to tensors of any sizes.
\end{proof}

\newpage

\subsection{Best Published results}
We report in Table~\ref{tab:ref} the references for each of the results in Table~2 in the article.
\begin{table}[h]
\begin{tabular}{cccc}
\toprule
Model & Metric & Result & Reference \\
\midrule
\multirow{2}{*}{WN18} & MRR & $0.94$ & \citet{trouillon_complex_2016} \\
{}                    & H@10& $\bm{0.97}$ & \citet{ma2017transt} \\
\\
\multirow{2}{*}{WN18RR} & MRR & $0.46$ & \citet{dettmers2017convolutional} \\
{}                      & H@10& $\bm{0.51}$ & \citet{dettmers2017convolutional} \\
\\
\multirow{2}{*}{FB15K} & MRR & $0.84$ & \citet{kadlec_knowledge_2017} \\
{}                     & H@10& $\bm{0.93}$ & \citet{shen2016implicit} \\
\\
\multirow{2}{*}{FB15K-237} & MRR & $0.32$ & \citet{dettmers2017convolutional} \\
{}                         & H@10& $0.49$ & \citet{dettmers2017convolutional} \\
\\
\multirow{2}{*}{YAGO3-10} & MRR & $0.52$ & \citet{dettmers2017convolutional} \\
{}                        & H@10& $0.66$ & \citet{dettmers2017convolutional} \\

\bottomrule
\end{tabular}
\caption{References for the \emph{Best Published} row in Table~\ref{tab:res}}
\label{tab:ref}
\end{table}
\end{document}